\documentclass{article}
\usepackage{graphicx} 
\usepackage{hyperref} 
\usepackage{amsmath, amsthm, amssymb}
\usepackage{tikz}
\usepackage{tcolorbox}
\tcbuselibrary{skins}
\usepackage{dashrule} 
\usepackage{xcolor}
\usepackage{wrapfig}

\makeatletter
\newcommand*{\rom}[1]{\expandafter\@slowromancap\romannumeral #1@}
\makeatother

\usepackage[ruled, lined, linesnumbered, longend]{algorithm2e}

\newcommand{\comp}{\mathrm{Comp}}

\newtheorem{theorem}{Theorem}

\newtheorem{definition}{Definition}

\newcommand{\sumtwo}{\mathrm{Sum}_2}
\newcommand{\pal}{\mathrm{Palindrome}}
\newcommand{\sign}{\mathrm{sign}}

\title{Lower bounds on transformers with infinite precision}

\author{ Alexander Kozachinskiy\footnote{alexander.kozachinskyi@cenia.cl}}
\date{Centro Nacional de Inteligencia Artificial, Chile
}

\begin{document}
\maketitle
\begin{abstract}
    In this note, we use the VC dimension technique to prove the first lower bound against one-layer softmax transformers with infinite precision. We do so for two tasks:  function composition, considered by Peng, Narayanan, and Papadimitriou, and the SUM$_2$ task, considered by Sanford, Hsu, and Telgarsky.
\end{abstract}
\section{Introduction}
Hahn~\cite{hahn2020theoretical} initiated the study of \emph{lower bounds} for the transformer architecture. A lower bound is a result of the form that for a certain task, with a certain number of parameters, there exists no choice of the parameters for which the transformer performs this task. Such results are meant to theoretically explain the poor performance of transformers at certain tasks.

Hahn obtained the first lower bounds against a theoretical model of transformers with \emph{hardmax} attention. In this model, instead of taking a convex combination of all input tokens using \emph{softmax}, one simply takes a single token where the attention is maximal. Hahn has shown that hardmax transformers with $O(1)$ layers are not able to compute parity, majority, and Dyck languages.  It was later proved by Hao, Angluin, and Frank~\cite{hao2022formal} that hardmax transformers with $O(1)$ layers are upper bounded by the complexity class $AC^0$. This allows to re-establish lower bounds of Hahn since parity, majority, and Dyck are well-known to be outside $AC^0$.

Proving lower bounds against $O(1)$-layer softmax transformers is notoriously hard; as was recently observed by Chen, Peng, and Wu~\cite{chen2024theoretical}, such transformers simulate constant-depth symmetric circuits, lower bounds for which seem to require a break-through in complexity theory. They have announced a lower bound against a constant number of layers of a \emph{decoder-only} architecture that avoids this barrier.

Previous works~\cite{peng2024limitations,DBLP:conf/nips/SanfordHT23,bhattamishra2024separations} have developed a technique against 1-layer softmax transformers where the output is computed in one of the tokens after the layer. The technique is based on \emph{communication complexity}; assuming input tokens are split between Alice and Bob, it is observed that the output token can be computed by Alice and Bob with low communication (and then the lower bound is proved via a reduction from some communication problem). 

This technique requires an assumption that computations in the transformers are performed with a relatively low number of \emph{bits of precision}. As a result, for a number of tasks, one can get lower bounds of the following form: any 1-layer softmax transformer, performing the task, requires either $n^{\Omega(1)}$ embedding dimension or $n^{\Omega(1)}$ number of precision bits.

\medskip

In this note, we develop a lower bound technique against 1-layer softmax transformers with \emph{infinite precision} (and with the output being computed in one of the tokens). Our technique employs upper bounds on \emph{VC dimension} of hypothesis classes, computable with a small number of basic arithmetic operations~\cite{goldberg1995bounding}. For our technique, we replace the assumption about the number of bits of precision with the assumption about the size of the \emph{output MLP} (and we assume it uses the ReLU activation). This is inevitable -- with infinite precision, a softmax layer can compute a binary representation of the input, and then a sufficiently large output MLP can compute any function.

We obtain lower bounds for two tasks, previously considered in the literature. One is \emph{function composition}~\cite{peng2024limitations}. On input for this task, we get a list $(\phi(1), \ldots, \phi(n))$ for a function $\phi\colon\{1, \ldots, n\}\to\{1, \ldots, n\}$. We are asked to output $\phi(\phi(1))$ (we show a lower bound even for a task of comparing $\phi(\phi(1))$ with 1). The other is the SUM$_2$ task, considered in~\cite{DBLP:conf/nips/SanfordHT23}  -- we get an array of $n$ integers, and we asked if there are two elements of the array whose sum is $0$. For both of these tasks, we show that any 1-layer softmax transformer for them must have either embedding dimension $n^{\Omega(1)}$ or the output MLP of size $n^{\Omega(1)}$.

We conclude the paper with a remark about the palindrome recognition task. It can be solved with constant embedding dimension and constant-size output MLP, assuming \emph{infinite precision} (more specifically, $O(n)$ bits would suffice). At the same time, via the communication complexity technique, it can be shown that with $n^{o(1)}$ bits of precision, one requires embedding dimension $n^{\Omega(1)}$.  The fact that our technique applies to SUM$_2$ but not to palindromes can be explained by a difference in the VC dimension of certain matrices.

\section{The model}
Fix an alphabet $\Sigma$ and the input length $n$. We are interested in computing functions of the form $f\colon \Sigma^n \to \{0, 1\}$. We consider three examples:
\begin{itemize}
    \item The composition function, denoted by $\comp_n$ for a given $n\in\mathbb{N}$. For this function, $\Sigma= \{1, \ldots, n\}$. To define $\comp_n(a_1, \ldots, a_n)$ for $(a_1, \ldots, a_n)\in\{1, \ldots, n\}^n$, we do the following. Let $\phi\colon\{1, \ldots, n\}\to\{1, \ldots, n\}$ be such that $\phi(1) = a_1, \ldots, \phi(n) = a_n$. Define 
    \[\comp_n(a_1\ldots a_n) = \begin{cases}
        1 & \phi(\phi(1)) = 1,\\
        0 & \text{otherwise}.
    \end{cases}\]
   \item For $n, m \in\mathbb{N}$, we set $\Sigma = \{-m, \ldots, m\}$ and define $\sumtwo^{n,m}\colon \Sigma^n\to\{0, 1\}$ by
\[\sumtwo^{n,m}(a_1, \ldots, a_n) = \begin{cases} 1& \text{there are } i, j\in\{1, \ldots, n\} \text{ with } a_i + a_j = 0,\\ 0 & \text{otherwise}.\end{cases}\]
\item For $n\in\mathbb{N}$, we define $\pal_n\colon\{0, 1\}^n\to\{0, 1\}$ by
\[\pal_n(x_1 \ldots x_n) = \begin{cases}1 & x_i = x_{n+1-i} \text{ for every } i\in\{1, \ldots, n\},\\ 0 & \text{otherwise}.\end{cases}\]
\end{itemize}

\begin{definition}
    A 1-layer single-token output transformer  $T$ of embedding dimension $d$ for input size $n$ and input alphabet $\Sigma$ is given by
    \begin{itemize}
        \item a ``positional encoding'' $p\colon \{1, \ldots, n\}\times \Sigma\to\mathbb{R}^d$
        \item initial ``value vector'' $h\in\mathbb{R}^d$ of the output token;
        \item the ``key'' and the ``query'' matrices $K, Q\in\mathbb{R}^{d\times d}$;
        \item the ``output MLP'' $\mathcal{N}\colon\mathbb{R}^d\to \mathbb{R}$ which is a ReLU neural network with some fixed choice of weights.
    \end{itemize} 
\end{definition}
The transformer $T$ defines a function $T\colon\Sigma^n\to\{0, 1\}$, computed on input $\bar x = x_1\ldots x_n\in \Sigma^n$ as follows. First, we define the value vectors of input tokens using positional encoding: 
\[f_1 = p(1,x _1), \ldots, f_n = p(n, x_n)\in\mathbb{R}^d.\]
Then we compute a convex combination of input tokens with weights, given by  the softmax of the  scalar product attention with matrices $K$ and $Q$:
\[\widehat{h} = \frac{e^{\langle K f_1, Q h\rangle}f_1 + \ldots + e^{\langle K f_n, Q h\rangle}f_n}{e^{\langle K f_1, Q h\rangle} + \ldots + e^{\langle K f_n, Q h\rangle}}.\]
We set $T(\bar x) = 1$ if and only if $\mathcal{N}(h + \widehat{h}) > 0$.

\section{Proofs}
\begin{theorem}
    There is no 1-layer single-token output transformer with embedding dimension $n^{o(1)}$ and output MLP with $n^{o(1)}$ ReLU neurons that computes $\comp_n$.
 \end{theorem}
 \begin{proof} Assume for contradiction that such transformer $T$ exists.
     We consider only inputs to $\comp_n$ of the form $(a_1, b_2, \ldots, b_n)$, where $a_1\in \{2, \ldots, n\}$ and $b_2, \ldots, b_n\in\{1, 2\}$. Observe that  $\comp_n(a_1 b_2\ldots b_n) = 1$ if and only if $b_{a_1} = 1$.

     Define value vectors of input tokens for such an input:
     \[f_1 = p(1, a_1),\,\, f_2 = p(2, b_2), \ldots,\,\, f_n = p(n, b_n).\]
Observe that:
\[\widehat{h} = \frac{e^{\langle K f_1, Q h\rangle}f_1 + \ldots + e^{\langle K f_n, Q h\rangle}f_n}{e^{\langle K f_1, Q h\rangle} + \ldots + e^{\langle K f_n, Q h\rangle}} = \frac{\bar x  + \bar y}{ p + q},\]
where
\begin{align*}
\bar x &= \bar x(a_1) =  e^{\langle K f_1, Q h\rangle}f_1\in\mathbb{R}^d,\\ p &= p(a_1) = e^{\langle K f_1, Q h\rangle}\in \mathbb{R}\\
\bar y &= \bar y(b_2\ldots b_n) = e^{\langle K f_2, Q h\rangle}f_2 + \ldots + e^{\langle K f_n, Q h\rangle}f_n\in\mathbb{R}^d\\
q &=  q(b_2\ldots b_n)  = e^{\langle K f_2, Q h\rangle} + \ldots + e^{\langle K f_n, Q h\rangle}\in\mathbb{R}.
\end{align*}
The output of the transformer is computed as 
\[T(a_1 b_2\ldots b_n) = \sign\left(\mathcal{N}(h + \frac{\bar x + \bar y}{p + q})\right),\]
where $\mathcal{N}$ is the output MLP.
Consider $F(\bar x, p, \bar y, q) = \sign\left(\mathcal{N}(h + \frac{\bar x + \bar y}{p + q})\right)$ as a parametric family of functions, where $(\bar x, p)$ are the inputs and $(\bar y, q)$ are parameters. It defines a hypothesis class:
\begin{equation}
\label{eq_hyp}
\left\{h_{\bar y, q}\colon\mathbb{R}^{d+1}\to\{0,1\} : h_{\bar y, q}(\bar x, p) = F(\bar x, p, \bar y, q),\,\, (\bar y, q)\in\mathbb{R}^{d+1}\right\}.
\end{equation}

Assuming for contradiction that $d = n^{o(1)}$ and the size of $\mathcal{N}$ is $n^{o(1)}$, we first show that \eqref{eq_hyp} has VC dimension $n^{o(1)}$. Indeed, it has $d + 1 = n^{o(1)}$ parameters, and $h_{\bar y, q}(\bar x, p)$ can be computed in $n^{o(1)}$ basic arithmetic operations and conditional jumps, based on comparing a real number with 0. Indeed, we need $d + 1$ additions to compute $\bar x + \bar y$ and $p + q$, then one division, and then $n^{o(1)}$  additions, products, and conditional jumps to calculate all ReLU neurons of $\mathcal{N}$.  By Theorem 2.3 in~\cite{goldberg1995bounding}, VC dimension is polynomial in these quantities (the number of parameters, the number of arithmetic operations and conditional jumps), which gives VC dimension $n^{o(1)}$.

We obtain a contradiction by showing that if $T$  computes $\comp_n$, the VC dimension of \eqref{eq_hyp} must be at least $n - 1$. Namely, we show that \eqref{eq_hyp} shatters the following $n - 1$ inputs:
\begin{align*}
(\bar x_1, p_1) &= (\bar x(2), p(2)) \\
&\vdots \\
(\bar x_{n-1}, p_{n-1}) &= (\bar x(n), p(n)) 
\end{align*}
For any $\delta\colon\{1, \ldots, n - 1\}\to\{0,1\}$, we show the existence of $(\bar y, q)\in\mathbb{R}^n$ such that:
\[h_{\bar y, q}(\bar x_{i}, p_{i}) = \delta(i), \qquad i = 1, \ldots, n -1.\]
Namely, set $b_{i+1} = 1$ if $\delta(i) = 1$ and $b_{i+1} = 2$ if $\delta(i) = 0$ for $i = 1, \ldots, n - 1$. Define $(\bar y, q) = (\bar y(b_2\ldots b_n), q(b_2\ldots b_n))$. Observe that for every $i\in\{1, \ldots, n - 1\}$, we have:
\begin{align*}
h_{\bar y, q}(\bar x_{i}, p_{i}) = T(i + 1, b_2\ldots b_n) = \begin{cases} 1 & b_{i + 1} = 1, \\ 0 & \text{otherwise},\end{cases} = \delta(i),
\end{align*}
as required.
 \end{proof}
\begin{theorem}
    There is no 1-layer single-token output transformer with embedding dimension $n^{o(1)}$ and output MLP with $n^{o(1)}$ ReLU neurons that computes $\sumtwo^{n,n}$.
\end{theorem}
\begin{proof}
Assume for contradiction that such transformer $T$ exists.
Denote $k = n/2$ and consider any two binary vectors $\alpha, \beta\in\{0, 1\}^k$. Define two sequence of integers $a_1, \ldots, a_k, b_1, \ldots, b_k\in\{-n, \ldots, n\}$ by:
\[a_i = \begin{cases}2i & \alpha_i = 1,\\ 1 & \alpha_i = 0,\end{cases}, \qquad b_i = \begin{cases}-2i & \beta_i = 1,\\ 1 & \beta_i = 0,\end{cases}\]
for $i = 1, \ldots, k$.
Observe that $\sumtwo^{n,n}(a_1\ldots a_k b_1\ldots b_k) = 1$ if and only if there exists $i \in\{1, \ldots, k\}$ such that $\alpha_i = \beta_i = 1$. We consider our transformer $T$ on inputs of the form $a_1\ldots a_k b_1\ldots b_k$. We start by writing the values of the input tokens:
  \[f_1 = p(1, a_1),\ldots,  f_k = p(2, a_k),\,\, f_{k+1} = p(k+1, b_1), \ldots,  f_n = p(n, b_n),\]
and then writing:
\[
\widehat{h} = \frac{e^{\langle K f_1, Q h\rangle}f_1 + \ldots + e^{\langle K f_n, Q h\rangle}f_n}{e^{\langle K f_1, Q h\rangle} + \ldots + e^{\langle K f_n, Q h\rangle}} = \frac{\bar x  + \bar y}{ p + q},\]

where
\begin{align*}
\bar x &= \bar x(\alpha) =  e^{\langle K f_1, Q h\rangle}f_1 + \ldots + e^{\langle K f_k, Q h\rangle}f_k\in\mathbb{R}^d,\\
 p &= p(\alpha) = e^{\langle K f_1, Q h\rangle}+ \ldots + e^{\langle K f_k, Q h\rangle}\in \mathbb{R}\\
\bar y &= \bar y(\beta) =  e^{\langle K f_{k+1}, Q h\rangle}f_{k+1} + \ldots + e^{\langle K f_n, Q h\rangle}f_n\in\mathbb{R}^d,\\
 q &=  q(\beta) =  e^{\langle K f_{k+1}, Q h\rangle} + \ldots + e^{\langle K f_n, Q h\rangle}\in\mathbb{R}.
\end{align*}
As in the previous proof, we define a parametric family of functions $F(\bar x, p, \bar y, q) =  \sign\left(\mathcal{N}(h + \frac{\bar x + \bar y}{p + q})\right)$, and consider a hypothesis class, based on it:
\begin{equation}
\label{eq_hyp2}
\left\{h_{\bar y, q}\colon\mathbb{R}^{d+1}\to\{0,1\} : h_{\bar y, q}(\bar x, p) = F(\bar x, p, \bar y, q),\,\, (\bar y, q)\in\mathbb{R}^{d+1}\right\}.
\end{equation}
By the same argument, its VC dimension is $n^{o(1)}$ based on the fact that $d$ and the size of $\mathcal{N}$ are $n^{o(1)}$. We now obtain a contradiction by showing that if $T$ computes $\sumtwo^{n, n}$, the VC dimension of \eqref{eq_hyp2} is at least $k = n/2$. Namely, we show that \eqref{eq_hyp2} must shatter the following $k$ inputs:
\begin{align*}
(\bar x_1, p_1) &= (\bar x(100\ldots 0), p(100\ldots 0)) \\
(\bar x_2, p_2) &= (\bar x(010\ldots 0), p(010\ldots 0)) \\
&\vdots \\
(\bar x_k, p_k) &= (\bar x(000\ldots 1), p(000\ldots 1)).
\end{align*}
Consider any binary word $\beta\in\{0, 1\}^k$ that we want to realize as the sequence of values on $ (\bar x_1, p_1), \ldots, (\bar x_k, p_k)$. Notice that $F(\bar x_i, p_i, \bar y(\beta), q(\beta))$ is equal to the value of the transformer on input, constructed as above from two binary vectors, one being the vector with the unique 1 at position $i$, and the other being $\beta$. By our construction, the output is $1$ if and only if $\beta_i = 1$. Thus, we notice that the sequence of values $\beta$ can be realized by the hypothesis $h_{\bar y, q}$ for $(\bar y, q) = (y(\beta), q(\beta))$.
\end{proof}

\begin{theorem}
    There exists a  1-layer single-token output transformer with embedding dimension $O(1)$ and output MLP with $O(1)$ ReLU neurons that computes $\pal_n$.
\end{theorem}
\begin{proof}[Proof sketch]
Assume on input we have a binary word $a_1\ldots a_k b_k\ldots b_1$. It is a palindrome if and only if $a_1 = b_1, \ldots, a_k = b_k$. It is easy to construct a softmax layer that computes a number, proportional to:
\[(a_1 - b_1)  + 10^{-1} (a_2 - b_2) + \ldots  +10^{-k +1} (a_k - b_k).\]
This number is $0$ if and only if the initial word is a palindrome, and this can be checked by a constant-size MLP.
\end{proof}

It is curious, why the VC dimension lower bound technique works for $\sumtwo$ but not for $\pal_n$. For both of these functions, lower bounds for $n^{o(1)}$-precision transformers can be proved via reductions from communication complexity -- from the \emph{disjointness} problem for $\sumtwo$ and from the \emph{equality} problem for $\pal_n$. The key factor is that the VC dimension of the \emph{communication matrix} of the disjointness problem is $n$ for $n$-bit strings, while for the equality problem it is low, just 1.

    \paragraph{Acknowledgment} Supported  by the National Center for Artificial Intelligence CENIA FB210017, Basal ANID. I thank Felipe Urrutia, Hector Jimenez, Tomasz Steifer, and Cristóbal Rojas for discussions.

\end{document}